\documentclass[twocolumn]{article}
\usepackage{geometry}
\usepackage{microtype}

\usepackage[utf8]{inputenc}
\usepackage[english]{babel}
\usepackage{graphicx}
\usepackage{amsfonts}
\usepackage{bbm}
\usepackage{enumitem}
\usepackage{amsmath}
\usepackage{amsthm}
\usepackage{algorithm}
\usepackage{algorithmic}
\usepackage{tabularx}
\usepackage{footnote}
\usepackage{caption}
\usepackage{bbm}
\usepackage[title]{appendix}
\usepackage{tikz}
\usetikzlibrary{positioning}
\usetikzlibrary{calc}
\usepackage{wrapfig}
\usepackage{subcaption}
\usepackage{float}
\usepackage[ddmmyyyy]{datetime}

\usepackage[affil-it]{authblk}
\usepackage{multirow}

\DeclareMathOperator*{\argmax}{arg\,max}
\DeclareMathOperator*{\argmin}{arg\,min}
\newcommand*\diff{\mathop{}\!\mathrm{d}}

\newcommand{\norm}[1]{\left\lVert #1 \right\rVert}

\newtheorem{theorem}{Theorem}
\newtheorem*{theorem*}{Theorem}
\newtheorem{corollary}{Corollary}[theorem]
\newtheorem{lemma}[theorem]{Lemma}

\title{
A Framework for Verification of Wasserstein Adversarial Robustness\\
}
\author{Tobias Wegel$^{1,2}$, Felix Assion$^2$, David Mickisch$^2$, Florens Greßner$^2$}
\affil{$^1$University of Göttingen\\ $^2$neurocat GmbH \\ }
\date{\vspace{-0.5cm}}

\begin{document}

\twocolumn[
  \begin{@twocolumnfalse}
    \maketitle
    \begin{abstract}
        Machine learning image classifiers are susceptible to adversarial and corruption perturbations. Adding imperceptible noise to images can lead to severe misclassifications of the machine learning model. Using $L_p$-norms for measuring the size of the noise fails to capture human similarity perception, which is why optimal transport based distance measures like the Wasserstein metric are increasingly being used in the field of adversarial robustness. Verifying the robustness of classifiers using the Wasserstein metric can be achieved by proving the absence of adversarial examples (certification) or proving their presence (attack). In this work we present a framework based on the work by Levine and Feizi \cite{levine2019wasserstein}, which allows us to transfer existing certification methods for convex polytopes or $L_1$-balls to the Wasserstein threat model. The resulting certification can be complete or incomplete, depending on whether convex polytopes or $L_1$-balls were chosen. Additionally, we present a new Wasserstein adversarial attack that is projected gradient descent based and which has a significantly reduced computational burden compared to existing attack approaches.\\
    \end{abstract}
\end{@twocolumnfalse}
]

\section{Introduction}
In recent years, machine learning systems -- in particular deep neural networks -- have been shown to be susceptible to adversarially crafted and randomly generated input noise. Since the publications of Biggio et al.\ \cite{Biggio_2013}, Szegedy et al.\ \cite{szegedy2014intriguing} and subsequently Goodfellow et al.\ \cite{goodfellow2015explaining}, this issue has become widely known and a large research community has formed to explain and explore the robustness of neural networks. To this day, image classification is the main focus of this research field. 

The choice of the imperceptibility metric, i.e.\ the measure which determines what counts as a ``small" perturbation or noise, plays an important role in the evaluation of robustness. Traditionally, $L_p$-norms were the natural choice, because they simplify computation complexity and machine learning researchers tend to think in terms of $L_p$-norms by default. However, using these metrics has a flaw: They are by no means specific to the kind of input that is provided. This has been acknowledged by large parts of the research community \cite{carlini2017evaluating, sharif2018suitability}. The $L_p$-norm does not use the structure of the images at all. An image largely relies on its two-dimensional nature and the distance between pixels often carries important information that is lost when using $L_p$-norms. In fact, it has been shown that $L_p$-norms are insufficient in replicating similarity perception of humans \cite{sharif2018suitability}. Thus, there have been efforts to include other distance measures in the robustness evaluation. One particular threat model for image classification that uses the Wasserstein distance as a metric in the input space has recently been introduced by Wong et al.\ \cite{wong2020wasserstein}. They introduce a computationally feasible method that produces adversarial examples with respect to the Wasserstein metric. The Wasserstein metric is a metric that actually uses the structure of images and coincides with the human intuition of similarity in images to some extent \cite{DBLP:dblp_journals/ijcv/RubnerTG00}. 

To defend machine learning models against Wasserstein-based attacks (i.e.\ robustify them), randomised smoothing has been transferred to be applicable to the Wasserstein distance by Levine and Feizi \cite{levine2019wasserstein}. The authors use an intriguing trick to change the domain in which the smoothing is done. This new domain contains all the flows - which represent a special kind of transport plan - between a fixed reference image and the other images. Therefore, for every image there is a fibre containing all the flows in the domain that are representations of transport plans between the reference and the image. Further, they show that the Wasserstein metric between two images is bounded by the $L_1$-norm in the new domain of any flows in the corresponding fibres of the images.
This result is based on Ling and Okada \cite{okada2007earthmovers} and therefore only applies to the 1-Wasserstein metric that uses the $L_1$-norm on the ground pixle space.
This domain shift will also be a central element for the verification methods in this paper.

In this work, we construct a certification method with respect to the Wasserstein distance. We utilize the results from Levine and Feizi \cite{levine2019wasserstein} by changing the domain in which the certification is done. We then show that our method is applicable to CNNs, RNNs and other model types which 1) have an affine transformation as the first operation and for which 2) a certification method with respect to $L_1$-norm bounded sets or convex polytopes already exists. To the best of our knowledge, it is the first certification method for the Wasserstein metric and novel in its approach. Additionally, we also construct an adversarial attack with respect to the Wasserstein metric, which utilizes the same domain shift. Overall, this domain shift together with the results presented in this paper facilitate a framework in which other verification procedures with respect to the Wasserstein metric can be conducted. 

\section{Preliminaries}
We consider gray-scale images in this work, but the results can partially be transferred to RGB images.
Gray-scale images can be mathematically described as a metric ground space $(\mathcal{X},d)=(\{1,...,n\}\times\{1,...,m\},d)$, where $(i,j)$ is the pixel in the $i^{th}$ row and the $j^{th}$ column and $d$ is the metric induced by the $L_1$-norm. The pixel intensity can then be specified by a probability distribution $\mu$ on $\mathcal{X}$, given by $\mu_{i,j}=\mu(\{(i,j)\})$, $(i,j)\in\mathcal{X}$, such that $\sum_{(i,j)\in\mathcal{X}}\mu_{i,j}=1$ and $\mu_{i,j}\geq0$ for all ${(i,j)\in\mathcal{X}}$. Usually, these constraints are not naturally fulfilled when looking at images. We will therefore assume that all images are normalised and denote the set of all probability distributions (i.e.\ gray-scale images) as $P(\mathcal{X})\subset \mathbb{R}^{n\times m}$.

Assume you have a classifier $F:\mathbb{R}^{n\times m}\rightarrow \{1,\dots,C\}$ that assigns the class $F(\mu)\in \{1,\dots,C\}$ to the image $\mu$. Consider a neighbourhood $\mathcal{S}\subset P(\mathcal{X})$ around $\mu$. An adversarial example with respect to $\mathcal{S}$ is an image $\nu\in \mathcal{S}$, such that $F(\nu)\neq F(\mu)$. An adversarial attack is a method that tries to compute a single adversarial example, whilst a certification method tries to prove the absence of adversarial examples in $\mathcal{S}$. Is the absence proven, $\mathcal{S}$ is said to be certified. Certification methods that can always prove this absence are called complete, whilst certification methods that might not be able to in some cases are called incomplete.

In this paper we consider the case of 
$$
\mathcal{S} = \{\nu\in P(\mathcal{X})| W_1(\mu,\nu) < \varepsilon\} =: B_\varepsilon^{W_1}(\mu).
$$
Here, $W_1$ denotes the 1-Wasserstein metric, which is defined as
\begin{equation}
    \label{eq:def_W1}
    W_1(\mu,\nu) := \inf_{\pi \in \Pi(\mu,\nu)}\int_\mathcal{X\times\mathcal{X}}\norm{x-y}_1\diff{\pi(x,y)}  
\end{equation}
with $\Pi(\mu,\nu)$ being the set of all couplings of $\mu$ and $\nu$.

\subsection{1-Wasserstein Distance}

To make (\ref{eq:def_W1}) more tractable, Levine and Feizi \cite{levine2019wasserstein} introduce objects called \textbf{flows} 
$$
\delta \equiv (\delta^{\rightarrow}, \delta^{\downarrow}) \in \mathbb{R}^{n\times(m-1) + m\times(n-1)}=: \mathcal{F}
$$ 
that are elements of the \textbf{flow domain} $\mathcal{F}$.
Further, they define the flow application map $\Delta: P(\mathcal{X})\times \mathcal{F}\rightarrow \mathbb{R}^{n\times m}$,
$$
\Delta(\mu, \delta)_{i,j} = \mu_{i,j}+\delta_{i-1,j}^{\downarrow}+\delta_{i,j-1}^{\rightarrow}-\delta_{i,j}^{\downarrow} -\delta_{i,j}^{\rightarrow}.
$$
For every pixel of an image a flow defines how much ``mass" is being moved to its neighbouring pixels, when the flow application map $\Delta$ is applied\footnote{For an intuition on what flows are, we recommend reading the paper by Levine and Feizi \cite{levine2019wasserstein}.}.
It can also be used to describe the 1-Wasserstein metric on $P(\mathcal{X})$:

\begin{theorem}[Levine and Feizi\ \cite{levine2019wasserstein}]
\label{theorem:levine}
For any probability distributions $\mu$ and $\nu$ on $\mathcal{X}$ there exists at least one flow $\delta$, such that $\nu=\Delta(\mu,\delta)$. Because of that, we can denote $\delta^{\mu}$ as an arbitrary flow, such that for a reference probability distribution $R$ it holds: $\Delta(R,\delta^{\mu})=\mu$. Denote the set that contains such $\delta^{\mu}$ by $S_R^{\mu}$. Furthermore it holds that
$$
W_1(\mu,\nu)=\min_{\nu=\Delta(\mu,\delta)}\norm{\delta}_1.
$$
\end{theorem}
The proof is based on a result by Ling and Okada \cite{okada2007earthmovers}. We will make use of theorem \ref{theorem:levine} as well as two consequences which have also been central for Levine and Feizi in \cite{levine2019wasserstein}.
First of all, let us fix a reference distribution $R$. Due to $\Delta(\Delta(\mu,\delta_1),\delta_2)=\Delta(\mu, \delta_1+\delta_2)$, it holds that
\begin{align*}
    W_1(\mu,\nu)&=\min_{\nu=\Delta(\mu,\delta)}\norm{\delta}_1\\ &=\min_{\nu=\Delta(\Delta(R,\delta^{\mu}),\delta)}\norm{\delta}_1\\
    &=\min_{\nu=\Delta(R,\delta^{\mu}+\delta)}\norm{\delta}_1\\ &=\min_{\nu=\Delta(R,\delta^{\nu})}\norm{\delta^{\mu}-\delta^{\nu}}_1\\
    &=\norm{\delta^{\mu}-\delta^{\nu}_*}_1
\end{align*}
for any $\delta^{\nu}_* \in \argmin_{\nu=\Delta(R,\delta^\nu)} \norm{\delta^{\mu}-\delta^{\nu}}_1$.
Secondly, we can do the following: Assume you have a classifier $F=c\circ f$, where $f:\mathbb{R}^{n\times m}\rightarrow \mathbb{R}^C$ has an affine transformation as the first operation and $c:\mathbb{R}^{C}\rightarrow \{1,...,C\}$, $c(x)\in \argmax_i x_i$. In particular, $F$ can be a FCNN or CNN. The machine learning model $F$ classifies the probability distributions $\mu$ on $\mathcal{X}$. We can define an equivalent classifier $\Tilde{F}$ that is defined as $\Tilde{F}(\delta)=(c\circ \Tilde{f})(\delta):=(c\circ f)(\Delta(R,\delta))$. It follows that $\tilde{f}(\delta^\mu)=f(\mu)$. Hence, we obtain a new classifier that classifies the flow in the same way as the original classifier classifies the corresponding probability distribution.

Levine and Feizi \cite{levine2019wasserstein} continue to use these results for smoothing, by smoothing $\Tilde{f}$ around a given $\delta^\mu$. In doing so, they use a method introduced by Lecuyer et al.\ in \cite{lecuyer2019certified}.

\subsection{The Application Map $\Delta$}

In order to derive a new verification framework for the Wasserstein metric from theorem \ref{theorem:levine} and its consequences, you first have to be able to map any $\mu$ to a flow $\delta^\mu\in S_R^{\mu}$. We therefore define the mapping $\Delta^{-1}:P(\mathcal{X})\times P(\mathcal{X})\rightarrow \mathcal{F}$, where $\Delta^{-1}(R,\mu)=\delta^\mu$. Since the flow $\delta^\mu$ is not unique, the inverse mapping of $\Delta$ is not well-defined and we have to select a suitable $\delta^\mu$ from the set $S_R^\mu$.
Therefore, for a fixed reference distribution $R$ the map $\Delta^{-1}(R,.)$ can be any right inverse of $\Delta(R,.)$, i.e.\ it has to have the following property: 
$$
(\Delta(R,.)\circ \Delta^{-1}(R,.))(\mu)=\mu.
$$
One choice can be made via the algorithm posed in the proof of Ling and Okada \cite{okada2007earthmovers}, which we summarise in algorithm \ref{mapping_mu_to_delta_algorithm_with_ot} in appendix B. 
It solves the optimal transport problem and uses the optimal coupling of $R$ and $\mu$ to construct $\delta^\mu$. This is particularly easy if we choose $R_{i,j}=1$ for any pixel $(i,j)$. However, the transport plan does not need to be optimal.

Having established an algorithm to construct the $\delta^\mu$, we can actually work with $\Tilde{f}$ in the flow domain $\mathcal{F}$ for the verification of Wasserstein adversarial robustness using certification and attack. To derive the general framework, we need two general statements about the flow application map $\Delta$.

\begin{theorem}[]
\label{Delta_is_affine_theorem}
The map $\delta \mapsto\Delta(R,\delta)$ is an affine transformation for any reference distribution $R$.
\end{theorem}
\begin{proof}
\setcounter{MaxMatrixCols}{20}
We want to express $\Delta(R,\delta)$ as $R+A\delta$, where we view  $R=(R_{1,1},R_{1,2},...,R_{n,m})^\top$
and 
$$
\delta = (\delta_{1,1}^{\rightarrow},\delta_{1,2}^{\rightarrow}...\delta_{1,m-1}^{\rightarrow}\delta_{2,1}^{\rightarrow}...\delta_{n,m-1}^{\rightarrow},\delta_{1,1}^{\downarrow},\delta_{1,2}^{\downarrow},...,\delta_{n-1,m}^{\downarrow})^\top
$$
as flattened column vectors and $A$ as a matrix of size $nm\times n(m-1)+m(n-1)$. 

$A$ can be given explicitly by a decomposition into two blocks $A=(A^{(1)}|A^{(2)})$, where $A^{(1)}\in \mathbb{R}^{nm\times n(m-1)}$, $A^{(2)}\in \mathbb{R}^{nm\times m(n-1)}$ and
\begin{align*}
    A^{(1)}_{i,j}&=
    \begin{cases}
    -1 &j=i-\lceil\frac{i}{m}\rceil+1, j\neq \lceil\frac{i}{m}\rceil(m-1)\\
    1 &j=i-\lceil\frac{i}{m}\rceil, j\neq(\lceil\frac{i}{m}\rceil-1)(m-1)-1 \\
    0 &\text{else}
    \end{cases}\\
    A^{(2)}_{i,j}&=
    \begin{cases}
    -1 &j=i\\
    1 &j=i-m\\
    0 &\text{else.}
    \end{cases}
\end{align*}
The fact that $\Delta(R,\delta)=R+A\delta$ then becomes a simple calculation.
\end{proof}

\begin{corollary}
\label{f_tilde_is_a_nn_corollary}
Assume $F=c\circ f$ is a classifier with an affine transformation $x \mapsto W_1x+b_1$ as its first operation. Then $\tilde{F}=c\circ \tilde{f}$ is also classifier with an affine transformation $\delta \mapsto \Tilde{W}_1\delta + \Tilde{b}_1$ as its first operation, where
$$
\Tilde{W_1}=W_1A \text{,\ \ \  } \Tilde{b}_1=W_1R+b_1.
$$
\end{corollary}

In particular, this means that if $F$ is for instance an FCNN, then $\Tilde{F}$ is also an FCNN.

\section{Certification}

In this section we will introduce two certification methods, the first called \textit{vanilla} and the second called \textit{fine-tuned}. The vanilla method is based on a $L_1$-ball certification in the flow domain. The fine-tuned version, on the other hand, takes into account that we can restrict the flow domain to flows that correspond to probability distributions, resulting in the certification of convex polytopes. The fine-tuned certification also optimises with respect to the chosen reference distribution. Because optimising over convex polytopes is computationally more complex than optimising over $L_1$-balls, choosing either method represents a trade-off between precision and computational complexity.

\subsection{Vanilla Certification}
\label{subsection:vanilla_certification}

We now come to the central result, which allows us to conduct certification of Wasserstein adversarial robustness. Theorem \ref{certification_in_the_flow_domain_theorem} utilizes the results that we have seen so far by showing that the certification task with respect to the Wasserstein metric on $P(\mathcal{X})$ can be solved by an $L_1$-certification in the flow domain.

\begin{theorem}[Certification in the Flow Domain]
\label{certification_in_the_flow_domain_theorem}
Let $F$ be a classifier and $\mu\in P(\mathcal{X})$ be a normalised gray-scale image. If we certify the absence of adversarial examples in $B_\varepsilon^{L_1}(\delta^{\mu})$ around any flow $\delta^{\mu}$ of $\mu$ for $\Tilde{F}$, then we have also certified the absence of adversarial examples in $B_\varepsilon^{W_1}(\mu)$ for F. 
\end{theorem}
\begin{proof}
Say we have $\mu\in P(\mathcal{X})$ and we want to certify the absence of adversarial examples within a radius of size $\varepsilon$ around $\mu$ measured by the Wasserstein metric. We want to show that $\forall \nu\in P(\mathcal{X})$ with $W_1(\mu,\nu)<\varepsilon$, it holds that $F(\mu)=F(\nu)$.\\
Choose any $\delta^{\mu}$ and $\nu$ such that $W_1(\mu,\nu)<\varepsilon$. Let 
$$
    \delta^{\nu}_* \in \argmin_{\nu=\Delta(R,\delta^{\nu})}\norm{\delta^{\mu}-\delta^{\nu}}_1.
$$
Then by theorem \ref{theorem:levine}, we know that $W_1(\mu,\nu)=\norm{\delta^{\mu}-\delta^{\nu}_*}_1$. It follows that $\norm{\delta^{\mu}-\delta^{\nu}_*}_1<\varepsilon$. By assumption it holds that $\forall \delta$ with $\norm{\delta^{\mu}-\delta}_1<\varepsilon$, we have $\Tilde{F}(\delta)=\Tilde{F}(\delta^{\mu})$ and thus it also holds for $\delta=\delta^{\nu}_*$. Consequently,
$$
F(\nu)=\Tilde{F}(\delta^{\nu}_*)=\Tilde{F}(\delta^{\mu})=F(\mu)
$$
which guarantees the absence of adversarial examples in $B_\varepsilon^{W_1}(\mu)$.
\end{proof}

Combining corollary \ref{f_tilde_is_a_nn_corollary} and theorem \ref{certification_in_the_flow_domain_theorem}, we actually get to use the large arsenal of certification methods that have been proven effective for $L_1$-certification of machine learning classifiers. For example, we can use FROWN by Lyu et al.\ \cite{lyu2019fastened} to conduct the actual certification. The \textbf{vanilla certification procedure} consists of the following steps:
\begin{enumerate}
    \item We select a reference distribution $R$, as well as an algorithm $\Delta^{-1}(R,.)$ (e.g.\ algorithm \ref{mapping_mu_to_delta_algorithm_with_ot} in appendix B) that maps any $\mu$ to a $\delta^\mu\in S_R^\mu$. 
    \item We transform the dataset $D$, which contains all the probability distributions $\mu$ to a data set $\Tilde{D}$, that contains the flows $\delta^\mu$ for every $\mu$. This is done by applying the right inverse mapping $\Delta^{-1}(R,.)$, i.e.\ with the chosen algorithm.
    \item We create the new classifier $\Tilde{F}$ by replacing the first layer weights and biases of $F$ according to corollary \ref{f_tilde_is_a_nn_corollary}.
    \item We pass $\Tilde{F}$ and $\Tilde{D}$ to an already existing $L_1$-certification method. The resulting bounds for the $L_1$-distance are also bounds for the Wasserstein distances for our original images and model (theorem \ref{certification_in_the_flow_domain_theorem}).
\end{enumerate}

\subsection{Fine-Tuned Certification}
\label{fine_tuning_chapter}

Before introducing the fine-tuned certification method, we will describe why a fine-tuning of the vanilla certification scheme is desirable. 

For one, if you fix a reference distribution $R$, not all elements in the flow domain $\mathcal{F}$ are feasible flows, i.e.\ not all flows result in probability distributions, when mapped by $\Delta(R,.)$. 
Suppose we have $n=m=28$, as is the case for MNIST images. Then the dimension of $\mathcal{F}$ is $n(m-1)+m(n-1)=2\times28\times27=1512$. Only a small subset of this high-dimensional space contains the feasible flows. Therefore, the probability that you randomly find a feasible flow for images with a resolution of $28\times 28$ like MNIST is close to zero.
If you go back to the proof of theorem \ref{certification_in_the_flow_domain_theorem}, you will see that we only need to show that $\Tilde{F}(\delta^\nu_*)=\Tilde{F}(\delta^\mu)$. However, if the $L_1$-ball around $\delta^\mu$ contains regions without any feasible flows, then we certify unnecessarily much and consequently we might not be able to determine the largest radius possible. Thus, by restricting the $L_1$-ball further, we can improve the accuracy of the vanilla certification approach. As we will see, whenever we use a complete certification method, no further fine-tuning is possible nor necessary. In appendix A we demonstrate the impact of fine-tuning with linear classifiers by solving the certification and attack task analytically.

A second option how we can fine-tune the certification is via the choice of the reference distribution $R$. As theorem \ref{certification_in_the_flow_domain_theorem} shows, we can choose any reference $R$ and then certify a $L_1$-ball around a $\delta^{\mu}\in S_R^\mu$. However, it is not obvious that the choice of $R$ has no influence on the certification result whenever we use an incomplete $L_1$-certification method. We will address this question empirically by conducting experiments with different reference distributions.

\subsubsection{Feasibility}

To get an understanding of where in the flow domain the feasible flows can be, we first have to recall what makes a flow feasible:
Fix a reference distribution $R$. For any flow $\delta\in\mathcal{F}$, the flow is feasible if and only if
\begin{enumerate}
    \item $\Delta(R,\delta)(\mathcal{X})= \norm{\Delta(R,\delta)}_1 = 1$
    \item $\Delta(R,\delta)(\{(i,j)\})\geq 0$ for all $(i,j)\in\mathcal{X}$.
\end{enumerate}
We denote the set of feasible flows $\mathbb{F}_R = \mathbb{F}\subset \mathcal{F}$, where $R$ is omitted if it is clear from the context.

Lemma \ref{first_characteristic_lemma} rules out characteristic one as a criterion for determining where feasible flows might lie in the flow domain. 

\begin{lemma}[First Characteristic]
\label{first_characteristic_lemma}
For any reference $R$, any flow $\delta\in\mathcal{F}$ fulfills characteristic one, i.e.\ it holds that $\Delta(R,\delta)(\mathcal{X})=1$.
\end{lemma}
\begin{proof}
Take any $R$ and any $\delta$. Due to theorem 1, we can write $\Delta(R,\delta)=R+A\delta$. Then it holds
\begin{align*}
    \Delta(R,\delta)(\mathcal{X}) 
    &= \sum_{(i,j)\in\mathcal{X}}\Delta(R,\delta)(\{(i,j)\})\\
    &= \sum_{k=1}^{nm} (R+A\delta)_k \\
    &= \sum_{k=1}^{nm}R_k+\sum_{l=1}^{nm}(A\delta)_l\\
    &= 1 + 1^\top A\delta
\end{align*}
where $1^\top:=(1,1,...,1)$. By inserting the explicit expression of the matrix $A$ (see theorem \ref{Delta_is_affine_theorem}), we get that $1^\top A=0^\top$ and thus we have $\Delta(R,\delta)(\mathcal{X})=1+0^\top\delta=1$, where $0^\top=(0,0,...,0)$.
\end{proof}
At first glance this result might seem a little bit surprising. However, it is actually rather obvious because applying any $\delta$ on a probability distribution resembles moving mass around the different pixels of an image. Whenever you remove mass at some point, you add the mass somewhere else, so the overall mass still has to sum to one (as also pointed out by Levine and Feizi \cite{levine2019wasserstein}). Consequently, lemma \ref{first_characteristic_lemma} implies that in order to restrict the $L_1$-ball around $\delta^\mu$, we have to rely on characteristic two of the feasible flow definition. 

\begin{lemma}[Second Characteristic]
\label{lemma:second_characteristic}
For any reference $R$, a flow $\delta\in\mathcal{F}$ is feasible, i.e.\ $\delta\in\mathbb{F}$, if and only if $R+A\delta \geq 0$ component-wise with $A$ defined as in theorem \ref{Delta_is_affine_theorem} and $R, \delta$ flattened column vectors. Hence,
$$
\mathbb{F} = \{\delta \in \mathcal{F} | A\delta + R \geq 0\}.
$$
\end{lemma}

\begin{proof}
Due to lemma \ref{first_characteristic_lemma}, we only have to check the second characteristic of the definition of a feasible flow. Furthermore, theorem \ref{Delta_is_affine_theorem} implies that
$$
\Delta(R,\delta)(\{(i,j)\}) = (A\delta + R)_k
$$
where $k$ is the index corresponding to the pixel $(i,j)$.
\end{proof}

Lemma \ref{lemma:second_characteristic} allows us to restrict the relevant certification area and thus allows us to increase the accuracy of our certification approach.

\begin{theorem}[Fine-tuned Certification in the Flow Domain]
\label{fine_tuned_certification_in_the_flow_domain_theorem}
Given a classifier $F$ and a data point $\mu\in P(\mathcal{X})$, if we certify $B_\varepsilon^{L_1}(\delta^{\mu}) \cap \{\delta | A\delta + R \geq 0\}$ around any flow $\delta^{\mu}$ of $\mu$ for $\Tilde{F}$, then we have certified $B_\varepsilon^{W_1}(\mu)$ for F. 
\end{theorem}

\begin{proof}
This follows from theorem \ref{certification_in_the_flow_domain_theorem} combined with lemma \ref{lemma:second_characteristic}. In the proof of theorem \ref{certification_in_the_flow_domain_theorem} we have seen that we only need to show that for any $\nu$ with \ $W_1(\mu, \nu)<\varepsilon$ it holds $\tilde{F}(\delta^\mu) = \tilde{F}(\delta^{\nu}_*)$. Because $\Delta(R,\delta^\nu_*) = \nu$, we know that $\delta^\nu_*\in \mathbb{F}$. By lemma \ref{lemma:second_characteristic} it is therefore sufficient to certify $B_\varepsilon^{L_1}(\delta^{\mu}) \cap \{\delta | A\delta + R \geq 0\}$.
\end{proof}

\begin{figure*}[!t]
    \centering
    \includegraphics[width=0.9\textwidth]{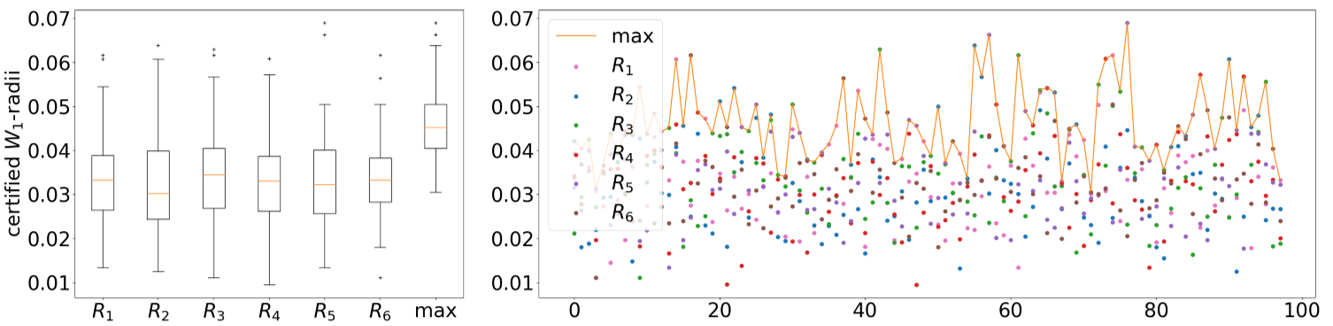}
    \caption{The certified $W_1$-radii for different references. The x-axis of the right figure denotes the different data points. The orange line marks the certificates received by taking the maximum over all references.}
    \label{fig:comparison_references}
\end{figure*}

It is important to note that $B_\varepsilon^{L_1}(\delta^{\mu}) \cap \{\delta | A\delta + R \geq 0\}$ is a \textbf{convex polytope}. This means that we can use any method suitable to certify convex polytopes, e.g.\ FROWN \cite{lyu2019fastened}.
The \textbf{fine-tuned certification procedure} consists of the same steps as the vanilla one, except that we do not pass $\tilde{D}$ and $\tilde{F}$ to an $L_1$-certification method, but instead a method that works with general convex polytopes. Also note that the fine-tuned certification method is optimal, in the sense that if we were to use a complete certification method, the resulting bounds for the Wasserstein metric would be exact. This can easily be seen: Assume the closest adversarial example $\nu$ to $\mu$ has Wasserstein distance $\varepsilon>0$. Then $\norm{\delta^\mu-\delta^\nu_*}_1 = \varepsilon$ and for all $\delta\in\mathbb{F}$, such that $\norm{\delta^\mu -\delta}_1< \varepsilon$, $\tilde{F}$ classifies $\delta$ correctly. Because the certification method for $\tilde{F}$ is complete, it will return $\varepsilon$.

\subsubsection{Reference Distribution}

For incomplete certification methods however, we have to address the second challenge of choosing a reference distribution $R$. We conducted an experiment with randomly sampled reference distributions and used algorithm \ref{mapping_mu_to_delta_algorithm_with_ot} to create the set of flows $\Tilde{D}$. We then used FROWN \cite{lyu2019fastened}, which is an incomplete certification method, to certify a three-layer FCNN on 100 images from MNIST. Note that for each reference, this only changes the bias in the first layer of $\Tilde{f}$ (see corollary \ref{f_tilde_is_a_nn_corollary}). In this way we can get a feeling about the impact of the choice of $R$ on the certification results. Overall, we considered six different reference distributions. We used two references $(R_1, R_2)$ where the values in each pixel were sampled uniformly in $[0,1]$. The references were then normalised to be probability distributions. Further, we considered two point references $(R_3, R_4)$ where all the probability mass was on one pixel (which was chosen randomly) and two images from MNIST as $(R_5, R_6)$. The results are shown in figure \ref{fig:comparison_references}. 

The average certification does not change a lot. As a matter of fact, the entire distributions of the certification results are very similar per reference. However, per data point, there are significant differences between the different reference distributions. This suggests that there is no systematically best choice of reference, but rather a best choice per data point. This leads to a very easy fine-tuning of our certification method with respect to the reference distribution. Let us denote the certificate of a data point $x$ that was received using a reference $R$ as $cert_R(x)$. Assume you randomly sample references $R_1,...,R_k$ (uniformly) or sample them from the given data set $D$, e.g.\ from MNIST. Then you can choose the fine-tuned certification $cert_F(x)$ as
$$
cert_F(x)= \max\{cert_{R_1}(x),...,cert_{R_k}(x)\}.
$$
In figure \ref{fig:comparison_references}, the $cert_F$ is denoted by $\max$. We see that we improve the certification substantially.\\

\subsection{Overview}

Alltogether we therefore introduced seven different kinds of certification methods, depending on which fine-tuning you choose and whether you use a complete certification method. 
Table \ref{tab:overview_certification_methods} summarises the different possible certification methods. Note that only the configuration using a complete convex polytope certifier and restricting the $L_1$-ball results in overall complete certification.

\begin{table}[]
    \centering
    \begin{tabular}{c|c|c|c||c}
          & Feas. & Ref. & Comp.Cert. & Comp.  \\
         \hline
         Vanilla & no & no & no & no\\
         \hline
         \multirow{6}{*}{Fine-tuned}& yes & no & no &no \\
         & no & yes & no & no\\
         \cline{2-5}
         & yes & yes & no &no\\
         \cline{2-5}
         & no & no & yes &no\\
         & no & yes & yes &no\\
         \cline{2-5}
         & yes & X & yes & yes
    \end{tabular}
    \caption{All possible certification methods using feasibility fine-tuning (Feas.), reference distribution fine-tuning (Ref.) and that use complete or incomplete certification methods (Comp.Cert.). The resulting method can either be complete or incomplete (Comp.). If the certification method is complete and feasibility fine-tuning is used, reference distribution fine-tuning cannot be applied (X).}
    \label{tab:overview_certification_methods}
\end{table}

We also compared the fine-tuned results using FROWN \cite{lyu2019fastened} of a three layer neural network with a smoothed classifier presented by Levine and Feizi \cite{levine2019wasserstein}, whose base classifier has a comparable accuracy on the test set of MNIST, i.e.\ around $98\%$. Interestingly, the \textit{certified} robustness of our model without smoothing is better than the \textit{certified} robustness with smoothing by a factor of around four, even though Levine and Feizi \cite{levine2019wasserstein} show that smoothing in general improves robustness. However, we do not claim that our robustness guarantees are better in general. When using smoothing to derive robustness guarantees, the depth of the neural network, or even the type of classifier, do not have any impact on such guarantees. Programs such as FROWN on the other hand scale very badly with depth. Therefore, our incomplete methods also scale poorly with the network depth.

\section{Attack}
\label{subsection:a_new_wasserstein_adversarial_attack}
We present a simple Wasserstein gradient descent attack, which avoids computationally demanding projections like Sinkhorn iterations.

The idea is that one can do projected gradient descent in the flow domain, where the projection is computationally less expensive than a projection with respect to the Wasserstein metric in the image domain, as proposed by Wong et al.\ \cite{wong2020wasserstein}. Assume we want to attack $F=c\circ f$ at the data point $\mu\in P(\mathcal{X})$ with label $y$. Choose any reference $R$ and let $\Tilde{F}$ and $\delta^\mu$ be defined as before. We want to find an adversarial example within a $W_1$-radius of $\varepsilon$. Define the loss\footnote{Actually, we can use any adversarial loss we like, but for projected gradient descent this loss is sufficient.}
$$
L(\delta)= \min_{t\neq y} \Tilde{f}_y(\delta)-\Tilde{f}_t(\delta)
$$
for any $\delta\in\mathcal{F}$ and let $\hat\delta\leftarrow\delta^\mu$. We can then perform a gradient step
$$
\hat\delta \leftarrow \hat\delta-\alpha\nabla L (\hat\delta).
$$
To ensure that $\Delta(R,\hat\delta)$ is a probability distribution and that $W_1(\Delta(R,\hat\delta),\mu)\leq\varepsilon$, we can project $\hat\delta$ onto $B_\varepsilon^{L_1}(\delta^\mu)\cap \mathbb{F}$. This amounts to the optimisation problem
$$
\hat\delta \leftarrow \text{proj}_{B_\varepsilon^{L_1}(\delta^\mu)\cap \mathbb{F}}(\hat\delta) \in \argmin_{\delta\in B_\varepsilon^{L_1}(\delta^\mu)\cap \mathbb{F}} \norm{\hat\delta-\delta}_1
$$
which, as we have seen in lemma \ref{lemma:second_characteristic}, is equivalent to
\begin{align*}
    &\min_{\delta\in\mathcal{F}} \norm{\hat\delta-\delta}_1 \\
    &\text{subject to\ \ } 0 \leq (A\delta+R)_i\ \forall\, i 
    \text{\ \ and \ } \norm{\delta^\mu-\delta}_1\leq \varepsilon.
\end{align*}

This optimisation problem is convex and can be solved easily with standard optimisation tools. 
We can repeat this as long as $\Tilde{F}(\hat\delta)=\Tilde{F}(\delta^{\mu})$. Once we have found $\hat\delta$ such that $\Tilde{F}(\hat\delta)\neq\Tilde{F}(\delta^{\mu})$, we can map it to $\Bar{\mu}=\Delta(R,\hat\delta)$ and consequently it will hold that $F(\Bar{\mu})\neq F(\mu)$ and $W_1(\Bar{\mu},\mu)\leq \norm{\hat\delta-\delta^\mu}_1\leq \varepsilon$ (theorem \ref{theorem:levine}).
Algorithm \ref{alg:new_wasserstein_adversarial_attack} summarises this attack. 

\begin{algorithm}[h!]
\caption{Wasserstein Projected Gradient Descent Attack}
\label{alg:new_wasserstein_adversarial_attack}
\hspace*{\algorithmicindent} \textbf{Input} Probability distribution $\mu$, reference distribution $R$, classifier $F$. \\
\hspace*{\algorithmicindent} \textbf{Output} Adversarial example $\Bar{\mu}$.\\
\begin{algorithmic}[1]
    \vspace{-0.5cm}
    \STATE $\delta^\mu \leftarrow \Delta^{-1}(R,\mu)$
    \STATE $\hat\delta \leftarrow \delta^\mu$
    \WHILE{a termination condition has not been met}
    \STATE $\hat\delta \leftarrow \hat\delta-\alpha \nabla L (\hat\delta)$
    \IF{$\hat\delta \notin\mathbb{F}$ or $\norm{\hat\delta-\delta^\mu}_1>\varepsilon$}
    \STATE $\hat\delta \leftarrow \text{proj}_{B_\varepsilon^{L_1}(\delta^\mu)\cap\mathbb{F}}(\hat\delta)$
    \ENDIF
    \ENDWHILE
    \STATE $\Bar{\mu}\leftarrow \Delta(R,\hat\delta)$
\end{algorithmic}
\end{algorithm}

An adversarial example created with this attack can be seen in figure \ref{fig:new_adversarial_example}. First experiments show that this attack produces adversarial examples with a $W_1$-perturbation size around 1.2 times larger than the attack introduced by Wong et al.\ \cite{wong2020wasserstein}. However the computation time is reduced by orders of magnitude. Setting the maximum number of iterations to 30 for both attacks, we decrease the average time per data point by a factor of around 160 compared to Wong's Wasserstein attack. However, a thorough evaluation has not yet been performed.

\begin{figure}[h]
    \centering
    \includegraphics[width=\linewidth]{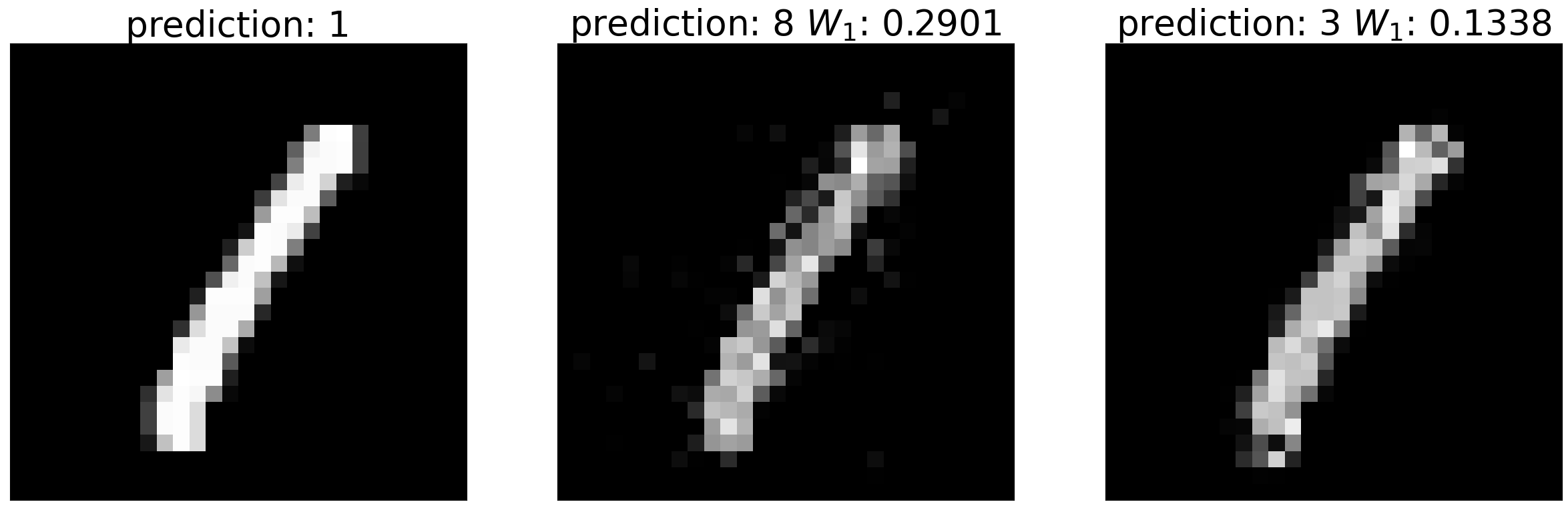}
    \caption{On the left: the original image. In the center: an adversarial example created with the attack introduced in this work. On the right: an adversarial example created with the method proposed by Wong et al.\ \cite{wong2020wasserstein}.}
    \label{fig:new_adversarial_example}
\end{figure}

\section{Conclusion}

In this paper we derive a certification method for Wasserstein adversarial robustness by constructing an equivalent $L_1$-certification problem in the flow domain (theorem \ref{certification_in_the_flow_domain_theorem}). In this $L_1$-certification problem we can then make use of the large arsenal of complete and incomplete certification methods that have already been developed in this threat model. To fully exploit this construction, we give insight into the set of all feasible flows which allows us to improve the certification even further. Based on these considerations, we also derive a new adversarial attack with respect to the Wasserstein metric (section \ref{subsection:a_new_wasserstein_adversarial_attack}). 
In future work, it is necessary to conduct large empirical evaluations for both presented methods. Whilst we only conducted experiments using FROWN \cite{lyu2019fastened}, for future experiments other certification methods should be considered as well. Also, one could attempt to certify with respect to versions of the Gromov-Wasserstein metric by combining this work with certification methods for geometric transformations.

\bibliographystyle{plain}
\bibliography{references}

\begin{thebibliography}{10}

\bibitem{Biggio_2013}
Battista Biggio, Igino Corona, Davide Maiorca, Blaine Nelson, Nedim Šrndić,
  Pavel Laskov, Giorgio Giacinto, and Fabio Roli.
\newblock Evasion attacks against machine learning at test time.
\newblock {\em Lecture Notes in Computer Science}, page 387–402, 2013.

\bibitem{carlini2017evaluating}
Nicholas Carlini and David Wagner.
\newblock Towards evaluating the robustness of neural networks, 2017.

\bibitem{goodfellow2015explaining}
Ian~J. Goodfellow, Jonathon Shlens, and Christian Szegedy.
\newblock Explaining and harnessing adversarial examples, 2015.

\bibitem{lecuyer2019certified}
Mathias Lecuyer, Vaggelis Atlidakis, Roxana Geambasu, Daniel Hsu, and Suman
  Jana.
\newblock Certified robustness to adversarial examples with differential
  privacy, 2019.

\bibitem{levine2019wasserstein}
Alexander Levine and Soheil Feizi.
\newblock Wasserstein smoothing: Certified robustness against wasserstein
  adversarial attacks, 2019.

\bibitem{okada2007earthmovers}
Haibin Ling and Kazunori Okada.
\newblock An efficient earth mover's distance algorithm for robust histogram
  comparison, 2007.

\bibitem{lyu2019fastened}
Zhaoyang Lyu, Ching-Yun Ko, Zhifeng Kong, Ngai Wong, Dahua Lin, and Luca
  Daniel.
\newblock Fastened crown: Tightened neural network robustness certificates,
  2019.

\bibitem{DBLP:dblp_journals/ijcv/RubnerTG00}
Yossi Rubner, Carlo Tomasi, and Leonidas~J. Guibas.
\newblock The earth mover's distance as a metric for image retrieval.

\bibitem{sharif2018suitability}
Mahmood Sharif, Lujo Bauer, and Michael~K. Reiter.
\newblock On the suitability of $l_p$-norms for creating and preventing
  adversarial examples, 2018.

\bibitem{szegedy2014intriguing}
Christian Szegedy, Wojciech Zaremba, Ilya Sutskever, Joan Bruna, Dumitru Erhan,
  Ian Goodfellow, and Rob Fergus.
\newblock Intriguing properties of neural networks, 2014.

\bibitem{wong2020wasserstein}
Eric Wong, Frank~R. Schmidt, and J.~Zico Kolter.
\newblock Wasserstein adversarial examples via projected sinkhorn iterations,
  2020.

\end{thebibliography}

\newpage
~
\newpage

\begin{appendices}
\section{Linear Classifiers}
One way to show the importance of taking into account which flows are actually feasible is by using a linear classifier $F=c\circ f \iff F(\mu)=c(W\mu+b)$. This setting is interesting, because then $\Tilde{F}=c\circ f\circ \Delta(R,.)$ is also a linear classifier (see corollary \ref{f_tilde_is_a_nn_corollary}). In this case, the certification and adversarial attack can be solved analytically by explicitly computing the minimal perturbation to the decision boundary.

\begin{lemma}[$L_1$ Perturbation Solution of Linear Classifiers]
\label{lemma:L1_linear_certification}
Let $F(\mu)=c(W\mu+b)$ be a linear classifier, where $W\in\mathbb{R}^{C\times nm}$ and $b\in\mathbb{R}^{C}$, and let $\mu$ be an image with label $y$. Take any reference distribution $R$ and for a matrix $B$ denote $B_i$ as the $i^{th}$ row of $B$. Then (without fine-tuning) the best certified $W_1$-radius for $\mu$ is
\begin{align*}
&\varepsilon = \min_{t\neq y} \min_j\\
&\left|\frac{(((WR+b)_t-(WR+b)_y)+((WA)_t-(WA)_y)\delta^\mu)}{((WA)_y-(WA)_t)_j}\right|.
\end{align*}
\end{lemma}
\begin{proof}
From corollary \ref{f_tilde_is_a_nn_corollary} we know that $\Tilde{F}$ is also a linear classifier with $\Tilde{f}(\delta)=WA\delta+WR+b=\Tilde{W}\delta+\Tilde{b}$. A flow $\delta$ is classified as $y$, if for all $t\neq y$ it holds
\begin{align*}
\Tilde{W}_y\delta + \Tilde{b}_y - \Tilde{W}_t\delta - \Tilde{b}_t \geq 0
&\iff(\Tilde{W}_y - \Tilde{W}_t)\delta + (\Tilde{b}_y-\Tilde{b}_t) \geq 0\\
&\iff: w_t^\top \delta +d_t \geq 0.
\end{align*}
The decision boundary is given by $\{\delta\in\mathcal{F}| w_t^\top \delta+d_t=0\}$. Therefore, we can certify a $L_1$-ball with radius $\varepsilon$ by
solving the perturbation to the decision boundary with respect to $\alpha_t$ along the standard basis vector $e_j$ for which $w_t^\top e_j$ is maximal ($j \in \argmax_j (w_t)_j$):
\begin{align*}
    w_t^\top(\delta^\mu+\alpha_{t}e_j)+d_t=0 
    \iff \alpha_{t}=\frac{(-d_t-w_t^\top \delta^\mu)}{w_t^\top e_j}.
\end{align*}
Plugging in $w_t^\top = \Tilde{W}_y-\Tilde{W}_t$, $d_t = \tilde{b}_y-\Tilde{b}_t$ and $\Tilde{W}=WA$, $\Tilde{b}=WR+b$ reveals
\begin{align*}
    \alpha_{t} &= \frac{(-(\tilde{b}_y-\Tilde{b}_t)-(\Tilde{W}_y-\Tilde{W}_t)\delta^\mu)}{(\Tilde{W}_y-\Tilde{W}_t)e_j}\\
    &= \frac{(-((WR+b)_y-(WR+b)_t)-((WA)_y-(WA)_t)\delta^\mu)}{((WA)_y-(WA)_t)_j}\\
    &= \frac{(((WR+b)_t-(WR+b)_y)+((WA)_t-(WA)_y)\delta^\mu)}{((WA)_y-(WA)_t)_j}.
\end{align*}

For $\varepsilon= \min_{t\neq y}|\alpha_{t}|$ we have provably the maximal certified vanilla radius for $B_\varepsilon^{L_1}(\delta^\mu)$ and by theorem \ref{certification_in_the_flow_domain_theorem} we have certified $B_\varepsilon^{W_1}(\mu)$.
\end{proof}

This lemma uses simple geometry to deduce the maximal certified radius, as can be seen in figure \ref{fig:linear_certification}. 
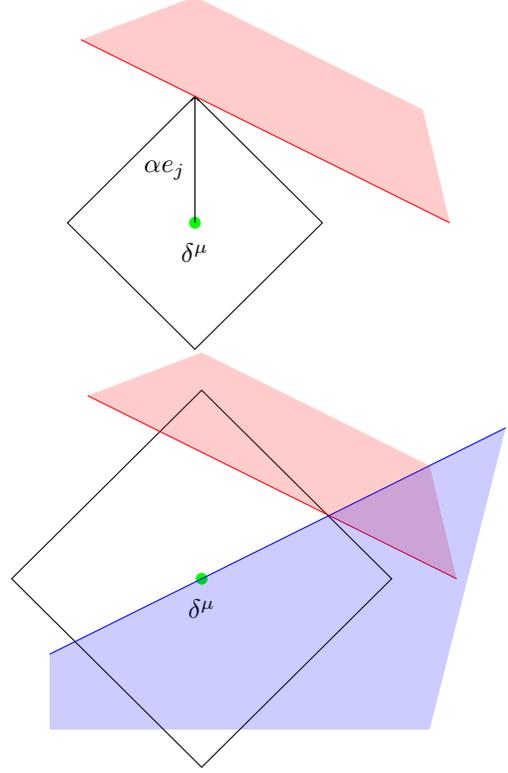
\begin{figure}[t]
    \centering
    \begin{subfigure}[b]{0.49\textwidth}
    \centering
    \begin{tikzpicture}
        \node[] (delta) at (0,0){};
        \node[] (K1) at (-0.4,0.7){$\alpha e_j$};
        \node[] (K2) at (0,-0.4){$\delta^\mu$};
        \filldraw[green] (delta) circle (2pt);
        \draw[] (1.67705,0)--(0,1.67705)--(-1.67705,0)--(0,-1.67705)--cycle;
        \draw[->] (0,0) -- (0,1.67705);
        \draw[red,-] (-1.5,2.42705) to (3.3541,0);
        
        \path[fill=red, opacity=0.2] (-1.5,2.42705)--(3.3541,0)--(3,1.5)--(0,3)--cycle;
    \end{tikzpicture}
    \end{subfigure}
    \begin{subfigure}[b]{0.49\textwidth}
    \centering
    \begin{tikzpicture}
        \node[] (delta) at (0,0){};
        \node[] (K2) at (0,-0.4){$\delta^\mu$};
        \draw[] (2.5,0)--(0,2.5)--(-2.5,0)--(0,-2.5)--cycle;
        \filldraw[green] (delta) circle (2pt);
        \draw[red,-] (-1.5,2.42705) to (3.3541,0);
        
        \draw[blue,-] (-2,-1) -- (4,2);
        
        \path[fill=red, opacity=0.2] (-1.5,2.42705)--(3.3541,0)--(3,1.5)--(0,3)--cycle;
        
        \path[fill=blue, opacity=0.2] (-2,-1) -- (4,2) -- (3,-2) -- (-2,-2)--cycle;
    \end{tikzpicture}
    \end{subfigure}
    \caption{Finding the best adversarial example for a linear classifier, on the top without restricting to the set of feasible flows, and on bottom using the fine-tuned version. The red region is the region where the classification changes, the blue region is the set of feasible flows.}
    \label{fig:linear_certification}
\end{figure}
It shows how to analytically find the best certificate for our vanilla certification method, in which we do not take into account where in the flow domain flows are feasible and where they are not. However, using theorem \ref{fine_tuned_certification_in_the_flow_domain_theorem}, we can improve upon this:

\begin{figure*}[t]
    \centering
    \includegraphics[width=0.9\textwidth]{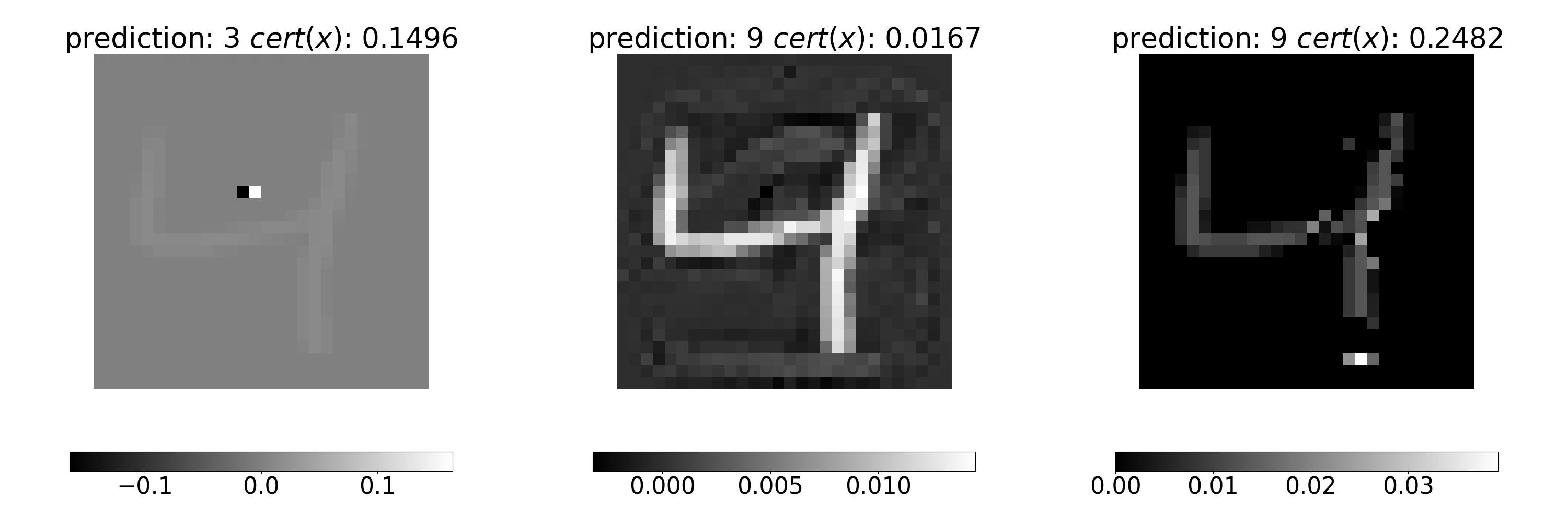}
    \caption{The adversarial examples constructed analytically for a linear classifier. On the left using the $L_1$-norm and thus corresponding to the vanilla certification method. In the middle using the $L_2$ norm. Note that the left two adversarial examples are not valid probability distributions. The $L_2$ certificate is worse by a factor of around 10. On the right the adversarial example that corresponds to the fine-tuned certification using lemma \ref{lemma:fine_tuned_linear_perturbation}. This adversarial example actually is a probability distribution.}
    \label{fig:linear_adversarial_examples}
\end{figure*}

\begin{lemma}[Fine-tuned Perturbation Solution of Linear Classifiers]
\label{lemma:fine_tuned_linear_perturbation}
Let
\begin{align*}
    &\Tilde{\alpha_t} = \min_\delta \norm{\delta-\delta^\mu}_1 \\
    \text{subject to\ \ }&\delta\in \mathbb{F}\cap \{\delta\in\mathcal{F} | w_t^\top \delta+d_t\le 0\}\\
    \iff &0 \leq A\delta+R \text{\ \ and\ \ } w_t^\top \delta+d_t \leq 0
\end{align*}
where $w_t$ is defined as in the proof of lemma \ref{lemma:L1_linear_certification}. Then it holds for $\Tilde{\varepsilon}=\min_t \Tilde{\alpha}_t$ that $\Tilde{\varepsilon} \geq \varepsilon$, where $\varepsilon$ is chosen according to lemma \ref{lemma:L1_linear_certification}. Furthermore, for
\begin{align*}
    &t,\delta \in \argmin_{t,\delta} \norm{\delta-\delta^\mu}_1 \\ \text{subject to\ \ } &0 \leq A\delta+R \text{\ \ and\ \ } w_t^\top \delta+d_t \leq 0,
\end{align*}
it holds that $\Delta(R,\delta)$ is the solution to
\begin{align*}
    &\Delta(R,\delta)\in\argmin_\nu W_1(\mu,\nu) \text{\ \ subject to\ \ } F(\nu)\neq y.
\end{align*}
\end{lemma}
\begin{proof}
The first part of this lemma follows directly from lemma \ref{lemma:second_characteristic} combined with theorem \ref{certification_in_the_flow_domain_theorem}. The second part can also be seen quite easily: Assume there exists another $\Tilde{\nu}$ such that $W_1(\mu, \Tilde{\nu})< W_1(\mu,\nu)$ and $F(\Tilde{\nu})\neq F(\mu)$. Then from theorem \ref{theorem:levine} it follows for some $\delta^{\Tilde{\nu}}$ that $\norm{\delta^{\mu}-\delta^{\Tilde{\nu}}}_1 < \norm{\delta^{\mu}-\delta^{\nu}}_1$. This is a contradiction, as $\delta^{\nu}$ is the minimum. Consequently, the assumption must be wrong.
\end{proof}

Because the above optimisation is a convex optimisation, it can be solved easily using standard optimisation tools. From the two lemmas above, we also get an explicit method how to compute the adversarial examples we certified against. For each method, an example can be seen in figure \ref{fig:linear_adversarial_examples}. It demonstrates really well, why the restriction to the set of feasible flows yields better results. Without doing so, one certifies against images that have nothing to do with probability distributions. Also, in this example the received certificate is approximately $1.6$ times better in the fine-tuned method. 
Additionally to being a proof of concept of a better certification method, the fine-tuned approach gives an extremely efficient method for finding adversarial examples for linear classifiers. 

\section{Algorithm for $\Delta^{-1}$}

Algorithm \ref{mapping_mu_to_delta_algorithm_with_ot} is based on the algorithm introduced by Ling and Okada \cite{okada2007earthmovers}. It uses the unique property of the $L_1$-norm that it can be decomposed into steps between neighbouring pixels. For example, if a transport plan $\pi$ tells you to move $\pi_{(i,j),(i+k,j)}$ mass from pixel $(i,j)$ to pixel $(i+k,j)$, we can also move the mass to $(i+1,j)$, then to $(i+2,j)$ and so on. This is stored in the flow $\delta$.

\begin{algorithm}[h!]
\caption{Mapping $\mu$ to $\delta^\mu\in S_R^\mu$ for general $R$}
\label{mapping_mu_to_delta_algorithm_with_ot}
\hspace*{\algorithmicindent} \textbf{Input} Probability distribution $\mu$, reference distribution $R$, optimal transport solver $OT$ that returns a transport plan\\
\hspace*{\algorithmicindent} \textbf{Output} A flow $\delta^\mu\in S_R^\mu$, such that $\Delta(R,\delta^{\mu})=\mu$
\begin{algorithmic}[1]
\STATE $\pi\leftarrow OT(R,\mu)$
\FOR{$(i,j)\in\mathcal{X}$}
    \FOR{$(k,l)\in\mathcal{X}$}
        \IF{$l\geq j$}
            \FOR{$step \in \{j,...,l-1\}$}
                \STATE $\delta^{\rightarrow}_{i,step}\leftarrow \delta^{\rightarrow}_{i,step}+ \pi_{(i,j),(k,l)}$
            \ENDFOR
        \ENDIF
        \IF{$l< j$}
            \FOR{$step \in \{l,...,j-1\}$}
                \STATE $\delta^{\rightarrow}_{i,step}\leftarrow \delta^{\rightarrow}_{i,step}- \pi_{(i,j),(k,l)}$
            \ENDFOR
        \ENDIF
        \IF{$k\geq i$}
            \FOR{$step \in \{i,...,k-1\}$}
                \STATE $\delta^{\rightarrow}_{step,l}\leftarrow \delta^{\rightarrow}_{step,l}+ \pi_{(i,j),(k,l)}$
            \ENDFOR
        \ENDIF
        \IF{$k< i$}
            \FOR{$step \in \{k,...,i-1\}$}
                \STATE $\delta^{\rightarrow}_{step,l}\leftarrow \delta^{\rightarrow}_{step,l}- \pi_{(i,j),(k,l)}$
            \ENDFOR
        \ENDIF
    \ENDFOR
\ENDFOR
\end{algorithmic}
\end{algorithm}

\end{appendices}
\end{document}